\documentclass{article}

\PassOptionsToPackage{numbers, compress}{natbib}
\usepackage[preprint]{neurips_2026}

\usepackage[utf8]{inputenc} 
\usepackage[T1]{fontenc}    
\usepackage{hyperref}       
\usepackage{url}            
\usepackage{booktabs}       
\usepackage{amsfonts}       
\usepackage{nicefrac}       
\usepackage{xcolor}         
\RequirePackage{graphicx}
\usepackage{xspace}
\usepackage{graphicx}
\usepackage{subfigure}
\usepackage{booktabs}
\usepackage[shortlabels]{enumitem}
\usepackage{pifont}
\usepackage[dvipsnames]{xcolor}

\RequirePackage{fancyhdr}
\RequirePackage{xcolor}
\RequirePackage{algorithm}
\RequirePackage{algorithmic}
\RequirePackage{natbib}
\RequirePackage{eso-pic}
\RequirePackage{forloop}
\RequirePackage{url}
\usepackage{hyperref}
\usepackage{booktabs}
\usepackage{amsmath,amssymb,amsthm}
\usepackage{soul}
\usepackage{subcaption}
\usepackage{multirow}
\usepackage{mathtools}
\usepackage{amsthm}
\usepackage[thinc]{esdiff}
\usepackage[table]{xcolor}
\usepackage{longtable}
\usepackage[capitalize,noabbrev]{cleveref}
\usepackage[textsize=tiny]{todonotes}
\setcitestyle{square}

\newcommand{\D}{\mathcal{D}}
\newcommand{\Cset}{\mathcal{C}}
\newcommand{\Z}{\mathcal{Z}}
\newcommand{\X}{\mathcal{X}}

\newtheorem{theorem}{Theorem}

\title{LiBaGS: Lightweight Boundary Gap Synthesis \\ for Targeted Synthetic Data Selection}

\author{%
  Abhishek Moturu \\ 
  Department of Computer Science\\
  University of Toronto\\
  The Hospital for Sick Children \\
  UHN KITE Research Institute \\
  T-CAIREM \\
  Vector Institute \\
  \texttt{moturuab@cs.toronto.edu} \\
  \And
  Anna Goldenberg $^*$  \\
  Department of Computer Science\\
  Department of Laboratory Medicine and Pathobiology \\
  University of Toronto\\
  The Hospital for Sick Children \\
  T-CAIREM \\
  Vector Institute \\
  \texttt{anna.goldenberg@utoronto.ca} \\
  \AND
  Babak Taati \thanks{equal contribution} \\
  Department of Computer Science\\
  Institute of Biomedical Engineering \\
  University of Toronto\\
  Rehabilitation Sciences Institute \\
  UHN KITE Research Institute \\
  Vector Institute \\
  \texttt{taati@cs.toronto.edu}
}
\begin{document}

\maketitle
\vspace{-1em}
\begin{abstract}
Synthetic data is useful only when the added samples fill missing parts of the training distribution that matter for the downstream task. We introduce \textbf{LiBaGS}, a lightweight, generator-agnostic method for targeted synthetic training data selection. LiBaGS scores candidate synthetic samples by combining decision-boundary proximity, predictive uncertainty, real-data density, and support validity, so that selected samples are both informative and likely to remain on the real data manifold. We then use a boundary-gap allocation rule that targets sparse but realistic decision-boundary neighborhoods, rather than simply adding more data or selecting only the most uncertain candidates. LiBaGS also learns when enough synthetic samples have been added through a marginal-value stopping rule, assigns softer labels near ambiguous boundaries, and uses a diversity objective to avoid redundant near-duplicate selections. Experiments show that LiBaGS improves accuracy over classical oversampling, hard augmentation, uncertainty-based selection, and targeted-generation selection criteria. 
\end{abstract}

\section{Introduction}
\label{introduction}
Synthetic data has become a common way to improve supervised learning when real data is scarce, expensive, imbalanced, private, or locally incomplete. Classical oversampling methods such as SMOTE, Borderline-SMOTE, ADASYN, Safe-Level-SMOTE, KMeans-SMOTE, MWMOTE, and Geometric SMOTE generate synthetic samples by interpolating minority examples, emphasizing boundary-adjacent points, or changing the local sampling geometry \citep{chawla2002smote,han2005borderline,he2008adasyn,bunkhumpornpat2009safe,last2017oversampling,barua2012mwmote,douzas2019geometric}. Representation-space methods such as manifold oversampling and DeepSMOTE move this idea into learned features, which is especially relevant for image-like data \citep{bellinger2018manifold,dablain2022deepsmote}. At the same time, augmentation methods such as mixup, Manifold Mixup, CutMix, Co-Mixup, CMO, and AugMax improve generalization by creating stronger or more diverse training examples \citep{zhang2017mixup,verma2019manifold,yun2019cutmix,kim2021comixup,park2022cmo,wang2021augmax}. These methods show that synthetic or augmented data can help, but they also highlight a recurring issue: \emph{not every synthetic sample is useful.}

This issue has become even more pronounced with modern generative models. GANs and diffusion models can produce large candidate pools, and recent work has shown that diffusion-generated data can improve image classification under the right conditions \citep{goodfellow2014gan,ho2020denoising,azizi2023synthetic}. Targeted-generation methods go further by using uncertainty, classifier feedback, long-tail signals, hard-negative objectives, or curriculum schedules to generate more informative synthetic samples \citep{niemeijer2024tsynd,hemmat2023feedback,koohpayegani2023genie,li2025gendataagent,askari2025improving,wang2024training,liang2025diffusion,hayden2025generative,kim2025generate}. They show that useful synthetic data is often hard, uncertain, rare, or classifier-informative. But, many methods still treat usefulness primarily as a local score, a generation prompt, or a curriculum signal. LiBaGS instead asks a much more explicit allocation question: \textbf{where should synthetic samples be placed, and how many should be kept, after accounting for real-data coverage?}

A candidate can be rare but irrelevant if it lies far from the decision boundary. A candidate can be uncertain but redundant if the training set already contains many nearby samples. A candidate can look boundary-like but harmful if it is off the real data support. Useful targeted synthetic data should satisfy three conditions at once: 
\begin{enumerate}
    \item it should lie near a decision-boundary neighborhood, 
    \item occupy a region where real data is locally sparse, and 
    \item remain on support, i.e. resemble realistic data from the true distribution.
\end{enumerate}

LiBaGS satisfies the above conditions. Let $\mathcal{Z}$ denote the feature space, let $z \in \mathcal{Z}$ denote the feature representation of an input sample, $p(z)$ the real data density, $q(z)$ the synthetic allocation density, $n$ the number of real samples, and $m$ the number of synthetic samples. Let $r(z)$ be a local boundary-gap value that is high only when a candidate is boundary-relevant, uncertain, and support-valid.

We view classification error near the decision boundary as depending on how much useful data exists in a region. This can be approximated by:
\begin{equation}
J_m(q)=\int_{\mathcal{Z}}\frac{r(z)}{np(z)+mq(z)}dz.
\label{eq:main_eq}
\end{equation}
Note that $np(z)+mq(z)$ represents the amount of local evidence from real and synthetic samples. Regions that already contain many real samples (large $p(z)$) benefit less from additional synthetic data. In contrast, regions with few real samples (small $p(z)$) but high boundary gap (large $r(z)$) is where synthetic samples can provide the greatest benefit.

Minimizing \ref{eq:main_eq} gives the following boundary-gap allocation rule:
\begin{equation}
q^*(z)
=
\frac{1}{m}\left[
\sqrt{\frac{r(z)}{\lambda}}
-
np(z)
\right]_+,
\label{eq:main_alloc}
\end{equation}
where $\lambda$ is a normalization constant automatically determined by the total synthetic allocation constraint. Intuitively, this rule places more synthetic samples in regions that are important for the decision boundary but poorly covered by real data, and places few or no synthetic samples in regions that already have enough real-data coverage.

LiBaGS is deliberately generator-agnostic and does not require retraining a diffusion model or training a new generator. A generator, interpolator, simulator, or augmentation pipeline proposes candidates, which LiBaGS then scores, filters, and diversifies. This is useful because generator quality varies. If a strong generator proposes many valid candidates, LiBaGS selects the most useful ones. If a weak generator proposes many artifacts, the support validity term and adaptive stopping rule prevent the method from blindly adding them. Thus, LiBaGS is a lightweight allocation rule that can be incorporated into existing synthetic data generating pipelines.

We make four contributions. First, we formulate targeted synthetic training data as a \emph{boundary-gap allocation problem}. Second, we derive a boundary-gap allocation rule that uses boundary proximity, uncertainty, support validity, and real-density coverage. Third, we provide a practical finite-candidate algorithm with adaptive stopping and diversity-aware selection. Fourth, we empirically compare LiBaGS against classical oversampling, hard augmentation, uncertainty-based selection, and targeted-generation selection criteria. Full proofs are in Appendix \ref{app:proofs}.

\section{Related Work}
\paragraph{Oversampling and interpolation.} SMOTE and its variants are classical baselines for synthetic training data \citep{chawla2002smote,han2005borderline,he2008adasyn,bunkhumpornpat2009safe,last2017oversampling,douzas2019geometric,barua2012mwmote}. These methods are often framed around class imbalance, local minority neighborhoods, or interpolation geometry. LiBaGS is motivated by a related intuition that samples near the decision boundary are often especially useful. However, it differs in two important ways: it does not target boundary proximity alone but also discounts regions that are already well covered by real data and it is generator-agnostic and can select from any proposed synthetic candidate pool, rather than being limited to interpolation among minority-class examples.

\paragraph{Augmentation and mixing.} Mixup, Manifold Mixup, CutMix, Co-Mixup, CMO, and AugMax improve generalization by creating similar, patch-level, context-rich, or hard augmented examples \citep{zhang2017mixup,verma2019manifold,yun2019cutmix,kim2021comixup,park2022cmo,wang2021augmax}. These methods provide strong augmentation baselines, however, LiBaGS is complementary to them since it can select from augmented candidates by asking whether a candidate fills an under-covered decision-boundary neighborhood.

\paragraph{Targeted synthetic generation.} TSynD optimizes latent representations toward high epistemic uncertainty for medical image classification \citep{niemeijer2024tsynd}. Feedback-guided synthesis uses classifier feedback while emphasizing support and diversity \citep{hemmat2023feedback}. GenDataAgent augments data on the fly using signals related to difficult training samples and gradient-update variance \citep{li2025gendataagent}. Deliberate Practice improves synthetic-data scaling by dynamically adding challenging examples \citep{askari2025improving}. DCDM and DisCL control sample difficulty or synthetic-to-real guidance levels \citep{wang2024training,liang2025diffusion}. Longtail Guidance and Generate What Matters target rare, hard, or classifier-useful synthetic samples \citep{hayden2025generative,kim2025generate}. Boundary-aware latent interpolation, on-manifold adversarial augmentation, targeted diffusion augmentation, and diffusion-based data augmentation are also related \citep{liu2025exploreaugment,patel2021manifold,nguyen2025we,islam2024diffusemix,wang2024diffmix,wang2025diffii}. But, LiBaGS contributes a density-aware allocation and stopping rule that can be applied after candidate generation.

\paragraph{Uncertainty, coverage, and active learning.} Active learning selects informative real examples using uncertainty or coverage \citep{settles2009active,sener2017active}. Query synthesis and adversarial active sampling methods synthesize or search near decision boundaries rather than only sampling existing points \citep{wang2015active,mayer2020adversarial}. Training dynamics methods identify hard, ambiguous, or forgotten examples \citep{toneva2018empirical,swayamdipta2020datamaps}. LiBaGS synthetic selection also includes support filtering, since some generated candidates may fall outside the real data distribution and therefore be unrealistic or unhelpful for training.

\section{Method}
\label{method}
\subsection{Setup and synthetic candidates}
Let $\D=\{(x_i,y_i)\}_{i=1}^n$ be a supervised training set and let $h:\X\to\Z$ be a fixed representation map. A candidate generator produces:
\begin{equation}
\Cset=\{(\tilde{x}_j,\tilde{y}_j)\}_{j=1}^{M},\quad z_j=h(\tilde{x}_j).
\end{equation}
The generator can be a simulator, local interpolation rule, augmentation pipeline, autoencoder, diffusion model, or any class-conditioned sampler \citep{kingma2013auto,shorten2019survey,goodfellow2014gan,ho2020denoising,islam2024diffusemix,wang2024diffmix,wang2025diffii,kim2024datadream}. LiBaGS only assumes access to the finite candidate pool. The goal is to select a subset $S\subseteq\Cset$ for final training.

For each candidate, LiBaGS estimates four quantities. First, a scoring model produces a class-probability vector $\pi(z)$, commonly used for predictive uncertainty, ensembles, and calibrated confidence scores in active learning / uncertainty estimation \citep{settles2009active,gal2016dropout,lakshminarayanan2017simple,guo2017calibration}. The top-two margin is:
\begin{equation}
    \Delta(z)=\pi_{(1)}(z)-\pi_{(2)}(z),
\end{equation}
where smaller values indicate proximity to the decision boundary. To emphasize these boundary-adjacent regions, LiBaGS defines the smooth decision boundary neighborhood weight:
\begin{equation}
a_\tau(z)=\exp\left(-\frac{\Delta(z)^2}{2\tau^2}\right),
\label{eq:boundary_weight}
\end{equation}
where $\tau>0$ controls how broadly points are considered near the decision boundary, with small $\tau$ meaning narrow focus and large $\tau$ meaning broader boundary neighborhood. In practice, $\tau$ is chosen from the empirical distribution of margin values (e.g., using a lower quantile of $\Delta(z)$ over the candidate pool), so that the weighting adapts to the amount of uncertainty in a given task.

The uncertainty score is the predictive entropy:
\begin{equation}
u(z)=H(\pi(z)).
\label{eq:uncertainty}
\end{equation}
The support-validity score $b(z)\in[0,1]$ is high for candidates close to the real representation region, following the general idea that synthetic samples should remain near the real data support \citep{hemmat2023feedback,wu2025filtering}. The real-density estimate $\hat p(z)$ is computed using nearest-neighbor or kernel density estimates \citep{cover1967nearest,silverman2018density}. The boundary-gap importance is then:
\begin{equation}
r(z)=a_\tau(z)u(z)b(z).
\label{eq:r}
\end{equation}

\subsection{Boundary-gap allocation score}

The continuous allocation rule in Eq.~\eqref{eq:main_eq} leads to the following finite candidate score:
\begin{equation}
G_j=\left[\sqrt{\frac{r(z_j)}{\lambda}}-n\hat p(z_j)\right]_+.
\label{eq:gap_score}
\end{equation}
Here, $r(z_j)$ measures how useful candidate $z_j$ is for improving the decision boundary, while $\hat p(z_j)$ measures how well that region is already covered by real training data. The score therefore favors candidates that are both boundary-important and underrepresented in the real dataset.

Intuitively, $G_j$ becomes large when a candidate lies in a high-uncertainty boundary region that still lacks sufficient real-data coverage, and becomes small when the candidate lies in a region that is already well covered, is far from the decision boundary, or appears off-support and therefore unlikely to provide useful training information.

\subsection{Adaptive stopping}

A fixed synthetic budget can be brittle, so LiBaGS instead uses an adaptive stopping rule related to diminishing-return selection and model-selection-by-stopping \citep{hoeffding1963probability,nemhauser1978analysis,krause2014submodular}. Intuitively, once a boundary region has already received enough synthetic coverage, adding additional nearby samples should provide progressively smaller gains.

Consider a local region of feature space with boundary importance $r_j$. Let $c_j$ denote the amount of existing real-data coverage in that region, and let $t_j$ denote the number of synthetic samples already allocated there. The marginal improvement from adding one more synthetic sample is modeled as:
\begin{equation}
\Delta_j(t_j)=\frac{r_j}{c_j+t_j}-\frac{r_j}{c_j+t_j+1}
=
\frac{r_j}{(c_j+t_j)(c_j+t_j+1)}.
\label{eq:marginal}
\end{equation}
As more synthetic samples are added to the same region, the marginal gain decreases. LiBaGS keeps adding candidates only while the best remaining marginal gain exceeds a threshold, $\eta>0$, which is chosen from the flattening point of the ordered marginal gain curve. Thus, the number of selected synthetic samples $\widehat m=|S|$ is learned from the candidate pool rather than fixed in advance.

\subsection{Support filtering, soft labels, and diversity}

A candidate with high uncertainty but a small support-validity score $b(z)$ is rejected unless it clears the stopping threshold. This is important when the generator is imperfect, since some boundary-like candidates may lie away from the real data support.

For labels, class-conditioned generators provide a proposed class $c(z)$. Near the boundary, hard labels can be overconfident, so we use the class-probability vector $\pi(z)$ to form a soft label:
\begin{equation}
\widetilde y(z)
=
\bigl(1-a_{\tau}(z)\bigr)e_{c(z)}
+
a_{\tau}(z)\pi(z),
\label{eq:soft_label}
\end{equation}
where $e_{c(z)}$ is the one-hot generator label and $\pi(z)$ is the class-probability vector from the scoring model. Further from the boundary, $a_{\tau}(z)$ is small, so the generator label is trusted more. Near the boundary, $a_{\tau}(z)$ is large, so the probability vector contributes more, reducing overconfidence in ambiguous regions.

Let $v_j=G_jb(z_j)$ and let $k(u,j)$ denote a similarity measure between candidates $u$ and $j$ in feature space, such as a Gaussian kernel or cosine similarity. For a selected set $S\subseteq\mathcal C$, LiBaGS maximizes:
\begin{equation}
F(S)=\sum_{u\in\mathcal C}v_u\max_{j\in S}k(u,j).
\label{eq:facility}
\end{equation}
This objective rewards a selected set that covers many high-value candidates while avoiding redundant near-duplicates.

\renewcommand{\algorithmicrequire}{\textbf{Input:}}
\renewcommand{\algorithmicensure}{\textbf{Output:}}

\begin{algorithm}[t]
\caption{Lightweight Boundary Gap Synthesis (LiBaGS)}
\label{alg:libags}
\begin{algorithmic}[1]
\REQUIRE real training set $\mathcal D$, candidate generator, encoder $h$, scoring model
\ENSURE final classifier, selected synthetic set $S$, and learned count $\widehat m$

\STATE Train the scoring model on the real data only.
\STATE Generate a candidate pool $\mathcal C$ and map candidates to $z=h(x)$.
\STATE For each candidate, estimate $\pi(z)$, $a_\tau(z)$, $u(z)$, $\hat p(z)$, and $b(z)$.
\STATE Compute the boundary-gap allocation score $G_j$ using Eq.~\eqref{eq:gap_score}.
\STATE Compute candidate values $v_j=G_jb(z_j)$.
\STATE Set the stopping threshold $\eta$ using early stopping on marginal-gain curve.
\STATE Select candidates greedily using Eq.~\eqref{eq:facility} and stop when the best marginal gain is below $\eta$.
\STATE Assign soft labels to selected candidates using Eq.~\eqref{eq:soft_label}.
\STATE Train the final classifier on $\mathcal D \cup S$.
\STATE \textbf{return} final classifier, selected synthetic set $S$, and $\widehat m=|S|$.
\end{algorithmic}
\end{algorithm}

\section{Theory}
\label{theory}
This section explains the main choices in Algorithm~\ref{alg:libags}. The main idea is that useful synthetic candidates should be near the decision boundary, uncertain, supported by real data, diverse, and not already well covered by the real training set.

\begin{theorem}[Local risk and boundary-gap allocation]
\label{thm:local_alloc}
Assume that, in a small region around $z$, the local sample count is proportional to $np(z)+mq(z)$, where $p$ is the real density and $q$ is the selected synthetic density. If the local error decreases with this count and is weighted by this: $r(z)=a_\tau(z)u(z)b(z)$, then the allocation objective is:
\begin{equation}
J_m(q)=
\int_{\mathcal Z}
\frac{r(z)}{np(z)+mq(z)}
\,dz.
\label{eq:local_objective}
\end{equation}
For fixed $m>0$, for some $\lambda>0$, minimizing $J_m(q)$ over $q\geq0$ with $\int q(z)\,dz=1$ gives:
\begin{equation}
q^*(z)
=
\frac{1}{m}\left[
\sqrt{\frac{r(z)}{\lambda}}
-
np(z)
\right]_+.
\label{eq:theory_alloc}
\end{equation}
\end{theorem}

\paragraph{Proof sketch.}
Regions with larger $np(z)+mq(z)$ have lower local error, while regions with larger $r(z)$ are more important for boundary learning. The objective is convex in $q$ and its optimal condition assigns synthetic mass until selected regions have equal marginal benefit. Subtracting the existing real coverage and clipping at zero gives Eq.~\eqref{eq:theory_alloc}; $G_j$ in Eq.~\eqref{eq:gap_score} is the finite-candidate version.

\begin{theorem}[Adaptive stopping]
\label{thm:adaptive}
For a small region $j$, let $c_j>0$ be real coverage, $t_j$ the number of selected synthetic samples, and $r_j\geq0$ the boundary importance. The gain of adding a sample is:
\begin{equation}
\Delta_j(t_j)
=
\frac{r_j}{c_j+t_j}
-
\frac{r_j}{c_j+t_j+1}
=
\frac{r_j}{(c_j+t_j)(c_j+t_j+1)}.
\label{eq:theory_marginal}
\end{equation}
Greedy selection with threshold $\eta$ accepts gains $> \eta$ and stops when best remaining gain is $< \eta$.
\end{theorem}

\paragraph{Proof sketch.}
The gain $\Delta_j(t_j)$ decreases as $t_j$ grows, so each region has diminishing returns. $\eta> 0$ is chosen from the flattening point of the ordered marginal gain curve. Greedy selection always takes the largest available gain and once that gain is below $\eta$, we stop.

\begin{theorem}[Soft-label stability]
\label{thm:soft}
Let $\rho(z)=\Pr(Y\,|\,Z=z)$ be the true class distribution. The soft label $\widetilde y(z)=\bigl(1-a_\tau(z)\bigr)e_{c(z)}+a_\tau(z)\pi(z)$ satisfies:
\begin{equation}
\|\widetilde y(z)-\rho(z)\|_1
\leq
\bigl(1-a_\tau(z)\bigr)\|e_{c(z)}-\rho(z)\|_1
+
a_\tau(z)\|\pi(z)-\rho(z)\|_1.
\label{eq:soft_error_bound}
\end{equation}
\end{theorem}

\paragraph{Proof sketch.}
The result follows directly from the triangle inequality applied to $\widetilde y(z)$.

\begin{theorem}[Diversity and support filtering]
\label{thm:diversity}
Let $F(S)=\sum_{u\in\mathcal C}v_u\max_{j\in S}k(u,j)$ and let $v_u=G_ub(z_u)$,
with $v_u\geq0$ and $k(u,j)\geq0$. Then, $F$ is monotone submodular.

Also, for a candidate $j$,
\begin{equation}
F(S\cup\{j\})-F(S)
\leq
\sum_{u\in\mathcal C}v_uk(u,j),
\end{equation}
so candidates with low support-weighted neighborhood value cannot pass the stopping threshold.
\end{theorem}

\paragraph{Proof sketch.}
Adding a candidate can only increase coverage, so $F$ is monotone. Adding a candidate also has diminishing returns, so $F$ is submodular. Since $v_u=G_ub(z_u)$, low-support regions have small value and are filtered unless they still clear the marginal gain threshold.

\section{Experiments}
\label{sec:experiments}

We evaluate LiBaGS in three settings: a two moons boundary-gap task, an 8$\times$8 handwritten-digit 3 vs.\ 8 task, and a CIFAR-10 cat vs.\ dog image task. Across all experiments, LiBaGS uses the same procedure described in Section~\ref{method}: a scoring model is trained only on the real training data, candidates are scored with $r(z)=a_\tau(z)u(z)b(z)$, the allocation score $G_j$ is computed, candidates are selected by the objective with marginal-gain stopping, and selected candidates are assigned the soft label in Eq.~\eqref{eq:soft_label}. The selected synthetic count $\widehat m$ is learned independently on each seed and is not fixed in advance. For fair comparison, all fixed-count baselines receive the same learned synthetic count $\widehat m$ as LiBaGS on a given seed.

\paragraph{Datasets and setup.}
The two moons experiment is a nonlinear classification problem, where we draw real training points away from a central boundary region, so the training set contains a local coverage gap near the decision boundary. The candidate pool contains simulator-generated boundary candidates, additional supported candidates, and off-support candidates. The representation map $h$ is a fixed random Fourier feature map, and the final classifier is logistic regression.

The 8$\times$8 digits experiment uses the scikit-learn handwritten digits data~\cite{digits}, licensed under (CC BY 4.0), restricted to a binary 3 vs.\ 8 task (since these two numbers may have a difficult decision boundary). The real training set contains a small number of easy examples per class, while the candidate pool contains held-out hard digits, same-class interpolations, and low-structure digits. This setting tests whether LiBaGS can use the boundary-gap score to identify useful ambiguous candidates while avoiding off-support candidates.

For CIFAR-10~\cite{cifar10}, which is widely used for research purposes, we use the cat vs.\ dog task with $25$ real training images per class, $1000$ auxiliary source images per class, and $800$ test images per class. We use a frozen ResNet-50~\cite{he2016deep} feature extractor from \texttt{timm}. The candidate pool is generated with Stable Diffusion v1-5 text/image-to-image model~\cite{Rombach_2022_CVPR} using $50$ candidates per class, diffusion strength $0.35$, guidance scale $6.0$, and $30$ denoising steps. The generator is used only to produce the finite candidate pool; LiBaGS then selects a subset adaptively.

\paragraph{Baselines.}
We compare against empirical risk minimization (ERM)~\citep{vapnik1999overview}, random candidate selection, noise augmentation, and the same selection-signal baselines used in the CIFAR experiment: TSynD epistemic criterion~\citep{niemeijer2024tsynd}, feedback-guided criterion~\citep{hemmat2023feedback}, GenDataAgent gradient-variance criterion~\citep{li2025gendataagent}, Deliberate Practice entropy criterion~\citep{askari2025improving}, conformal filtering~\citep{wu2025filtering}, long-tail rare/hard criterion~\cite{hayden2025generative}, C2I boundary-influence criterion~\citep{kim2025generate}, and uncertainty-only~\citep{settles2009active}. On the two small datasets, we additionally include classical oversampling and augmentation baselines: SMOTE~\citep{chawla2002smote}, Borderline-SMOTE~\citep{han2005borderline}, ADASYN~\citep{he2008adasyn}, Safe-Level-SMOTE~\citep{bunkhumpornpat2009safe}, KMeans-SMOTE~\citep{last2017oversampling}, DeepSMOTE~\citep{dablain2022deepsmote}, and AugMax~\cite{wang2021augmax}. These additional baselines do not use the external candidate pool, but synthesize examples directly from the real training set.

\begin{table}[h!]
\centering
\caption{Accuracy across the three datasets across baselines used in the CIFAR-10 experiment. All non-ERM methods use the adaptive synthetic count learned by LiBaGS on each seed; means and standard deviations are over 5 seeds.}
\label{tab:main_results}
\small
\begin{tabular}{lccc}
\toprule
Method & Two moons & 8x8 digits & CIFAR-10 cat/dog \\
\midrule
ERM~\cite{vapnik1999overview} & 0.8859 $\pm$ 0.0165 & 0.9391 $\pm$ 0.0167 & 0.8040 $\pm$ 0.0230 \\
Noise augmentation & 0.8977 $\pm$ 0.0226 & 0.9342 $\pm$ 0.0189 & 0.7918 $\pm$ 0.0267 \\
Random candidates & 0.9299 $\pm$ 0.0198 & 0.9491 $\pm$ 0.0167 & 0.8120 $\pm$ 0.0159 \\
TSynD epistemic criterion~\cite{niemeijer2024tsynd} & 0.9021 $\pm$ 0.0313 & 0.9503 $\pm$ 0.0158 & 0.8054 $\pm$ 0.0173 \\
Feedback-guided criterion~\cite{hemmat2023feedback} & 0.9330 $\pm$ 0.0184 & 0.9516 $\pm$ 0.0265 & 0.8078 $\pm$ 0.0158 \\
GenDataAgent gradient-variance criterion~\cite{li2025gendataagent} & 0.9097 $\pm$ 0.0840 & 0.9342 $\pm$ 0.0306 & 0.8114 $\pm$ 0.0150 \\
Deliberate Practice entropy criterion~\cite{askari2025improving} & 0.9437 $\pm$ 0.0169 & 0.9466 $\pm$ 0.0303 & 0.8101 $\pm$ 0.0140 \\
Conformal filtering~\cite{wu2025filtering} & 0.9209 $\pm$ 0.0254 & 0.9466 $\pm$ 0.0273 & 0.8130 $\pm$ 0.0151 \\
Long-tail rare/hard criterion~\cite{hayden2025generative} & 0.9369 $\pm$ 0.0197 & 0.9478 $\pm$ 0.0143 & 0.8088 $\pm$ 0.0167 \\
C2I boundary-influence criterion~\cite{kim2025generate} & 0.9429 $\pm$ 0.0170 & 0.9404 $\pm$ 0.0222 & 0.8141 $\pm$ 0.0149 \\
Uncertainty-only~\cite{settles2009active} & 0.9431 $\pm$ 0.0166 & 0.9391 $\pm$ 0.0178 & 0.8084 $\pm$ 0.0199 \\
\midrule
LiBaGS & \textbf{0.9454 $\pm$ 0.0163} & \textbf{0.9540 $\pm$ 0.0315} & \textbf{0.8159 $\pm$ 0.0167} \\
\bottomrule
\end{tabular}
\end{table}

\paragraph{Main results.}
Table~\ref{tab:main_results} reports accuracy for the shared baselines across all three datasets. LiBaGS obtains the best mean accuracy in each setting. On two moons, LiBaGS improves over ERM from $0.8859$ to $0.9454$, showing that the boundary-gap allocation score recovers useful samples from the missing boundary region. On 8$\times$8 digits 3 vs.\ 8, LiBaGS improves over ERM from $0.9391$ to $0.9540$, while also outperforming the other baselines. On CIFAR-10 cat vs.\ dog, LiBaGS reaches $0.8159$, compared with $0.8040$ for ERM. These gains are modest but consistent, and they match the intended use case: LiBaGS is most useful when the candidate pool contains boundary-adjacent samples that are informative, on-support, and not already well covered by the real data. Together, the results suggest that explicitly targeting sparse boundary regions can improve synthetic-data selection without requiring a fixed (unknown) synthetic budget.

Tables~\ref{tab:two_moons} and~\ref{tab:digits8x8} give the full results for the two moons experiment and the 8$\times$8 digits experiment, including the AUROC. The two moons experiment shows the largest absolute improvement because the missing boundary region is explicit. Figure~\ref{fig} shows this explicitly. Although both tasks have several methods obtaining very similar ranking performance, LiBaGS is still the strongest on accuracy, while some uncertainty baselines have slightly higher AUROC. This suggests that LiBaGS is most useful as a targeted training-data selector: rather than simply improving global ranking of synthetic samples, it adds training signal in sparse boundary regions where additional examples can directly improve the learned classifier.

\begin{table}[h!]
\centering
\small
\caption{Two moons results. LiBaGS learns the synthetic count (177.2 $\pm$ 42.5 across 5 seeds) by marginal-gain stopping; all baselines use this count for fair selection comparison.}
\label{tab:two_moons}
\begin{tabular}{lcc}
\toprule
Method & Accuracy & AUROC \\
\midrule
ERM~\cite{vapnik1999overview} & 0.8859 $\pm$ 0.0165 & 0.9554 $\pm$ 0.0124 \\
Noise augmentation & 0.8977 $\pm$ 0.0226 & 0.9626 $\pm$ 0.0149 \\
Random candidates & 0.9299 $\pm$ 0.0198 & 0.9800 $\pm$ 0.0102 \\
SMOTE~\cite{chawla2002smote} & 0.8960 $\pm$ 0.0197 & 0.9604 $\pm$ 0.0152 \\
Borderline-SMOTE~\cite{han2005borderline} & 0.9117 $\pm$ 0.0212 & 0.9671 $\pm$ 0.0145 \\
ADASYN~\cite{he2008adasyn} & 0.9083 $\pm$ 0.0172 & 0.9652 $\pm$ 0.0145 \\
Safe-Level-SMOTE~\cite{bunkhumpornpat2009safe} & 0.8939 $\pm$ 0.0145 & 0.9584 $\pm$ 0.0101 \\
KMeans-SMOTE~\cite{last2017oversampling} & 0.8957 $\pm$ 0.0182 & 0.9598 $\pm$ 0.0143 \\
DeepSMOTE~\cite{dablain2022deepsmote} & 0.8967 $\pm$ 0.0202 & 0.9601 $\pm$ 0.0144 \\
AugMax~\cite{wang2021augmax} & 0.9086 $\pm$ 0.0198 & 0.9660 $\pm$ 0.0153 \\
TSynD epistemic criterion~\cite{niemeijer2024tsynd} & 0.9021 $\pm$ 0.0313 & 0.9615 $\pm$ 0.0239 \\
Feedback-guided criterion~\cite{hemmat2023feedback} & 0.9330 $\pm$ 0.0184 & 0.9799 $\pm$ 0.0137 \\
GenDataAgent gradient-variance criterion~\cite{li2025gendataagent} & 0.9097 $\pm$ 0.0840 & 0.9666 $\pm$ 0.0475 \\
Deliberate Practice entropy criterion~\cite{askari2025improving} & 0.9437 $\pm$ 0.0169 & \textbf{0.9857 $\pm$ 0.0083} \\
Conformal filtering~\cite{wu2025filtering} & 0.9209 $\pm$ 0.0254 & 0.9725 $\pm$ 0.0194 \\
Long-tail rare/hard criterion~\cite{hayden2025generative} & 0.9369 $\pm$ 0.0197 & 0.9831 $\pm$ 0.0112 \\
C2I boundary-influence criterion~\cite{kim2025generate} & 0.9429 $\pm$ 0.0170 & 0.9855 $\pm$ 0.0080 \\
Uncertainty-only~\cite{settles2009active} & 0.9431 $\pm$ 0.0166 & 0.9855 $\pm$ 0.0083 \\
\midrule
LiBaGS (ours) & \textbf{0.9454 $\pm$ 0.0163} & 0.9853 $\pm$ 0.0082 \\
\bottomrule
\end{tabular}
\end{table}

\begin{table}[h!]
\centering
\small
\caption{8$\times$8 digits results on classes 3 vs.\ 8. LiBaGS learns the synthetic count (47.4 $\pm$ 10.2 across 5 seeds) by marginal-gain stopping; all baselines use this count for fair selection comparison.}
\label{tab:digits8x8}
\begin{tabular}{lcc}
\toprule
Method & Accuracy & AUROC \\
\midrule
ERM~\cite{vapnik1999overview} & 0.9391 $\pm$ 0.0167 & 0.9905 $\pm$ 0.0049 \\
Noise augmentation & 0.9342 $\pm$ 0.0189 & 0.9896 $\pm$ 0.0056 \\
Random candidates & 0.9491 $\pm$ 0.0167 & 0.9845 $\pm$ 0.0075 \\
SMOTE~\cite{chawla2002smote} & 0.9441 $\pm$ 0.0237 & 0.9915 $\pm$ 0.0059 \\
Borderline-SMOTE~\cite{han2005borderline} & 0.9404 $\pm$ 0.0227 & 0.9895 $\pm$ 0.0065 \\
ADASYN~\cite{he2008adasyn} & 0.9478 $\pm$ 0.0209 & 0.9905 $\pm$ 0.0053 \\
Safe-Level-SMOTE~\cite{bunkhumpornpat2009safe} & 0.9453 $\pm$ 0.0221 & 0.9896 $\pm$ 0.0067 \\
KMeans-SMOTE~\cite{last2017oversampling} & 0.9453 $\pm$ 0.0161 & 0.9908 $\pm$ 0.0055 \\
DeepSMOTE~\cite{dablain2022deepsmote} & 0.9366 $\pm$ 0.0217 & 0.9893 $\pm$ 0.0066 \\
AugMax~\cite{wang2021augmax} & 0.9404 $\pm$ 0.0251 & 0.9906 $\pm$ 0.0060 \\
TSynD epistemic criterion~\cite{niemeijer2024tsynd} & 0.9503 $\pm$ 0.0158 & 0.9859 $\pm$ 0.0080 \\
Feedback-guided criterion~\cite{hemmat2023feedback} & 0.9516 $\pm$ 0.0265 & 0.9860 $\pm$ 0.0203 \\
GenDataAgent gradient-variance criterion~\cite{li2025gendataagent} & 0.9342 $\pm$ 0.0306 & 0.9845 $\pm$ 0.0106 \\
Deliberate Practice entropy criterion~\cite{askari2025improving} & 0.9466 $\pm$ 0.0303 & 0.9911 $\pm$ 0.0072 \\
Conformal filtering~\cite{wu2025filtering} & 0.9466 $\pm$ 0.0273 & 0.9856 $\pm$ 0.0248 \\
Long-tail rare/hard criterion~\cite{hayden2025generative} & 0.9478 $\pm$ 0.0143 & 0.9914 $\pm$ 0.0063 \\
C2I boundary-influence criterion~\cite{kim2025generate} & 0.9404 $\pm$ 0.0222 & 0.9904 $\pm$ 0.0065 \\
Uncertainty-only~\cite{settles2009active} & 0.9391 $\pm$ 0.0178 & \textbf{0.9917 $\pm$ 0.0051} \\
\midrule
LiBaGS (ours) & \textbf{0.9540 $\pm$ 0.0315} & 0.9873 $\pm$ 0.0104 \\
\bottomrule
\end{tabular}
\end{table}

\begin{figure}[h!]
    \centering
    \includegraphics[width=0.54\linewidth]{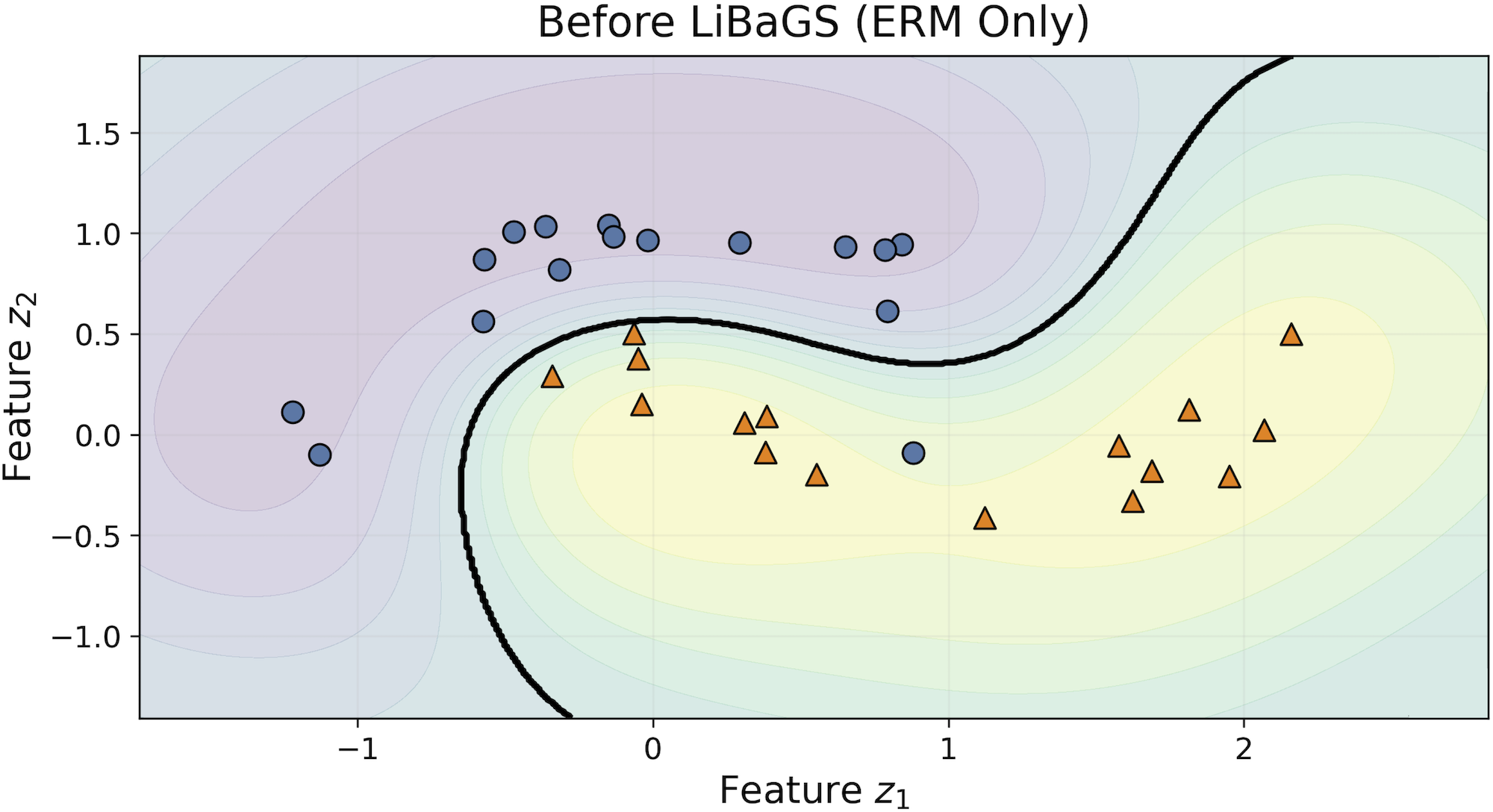} \\
    \vspace{1em}
    \includegraphics[width=0.54\linewidth]{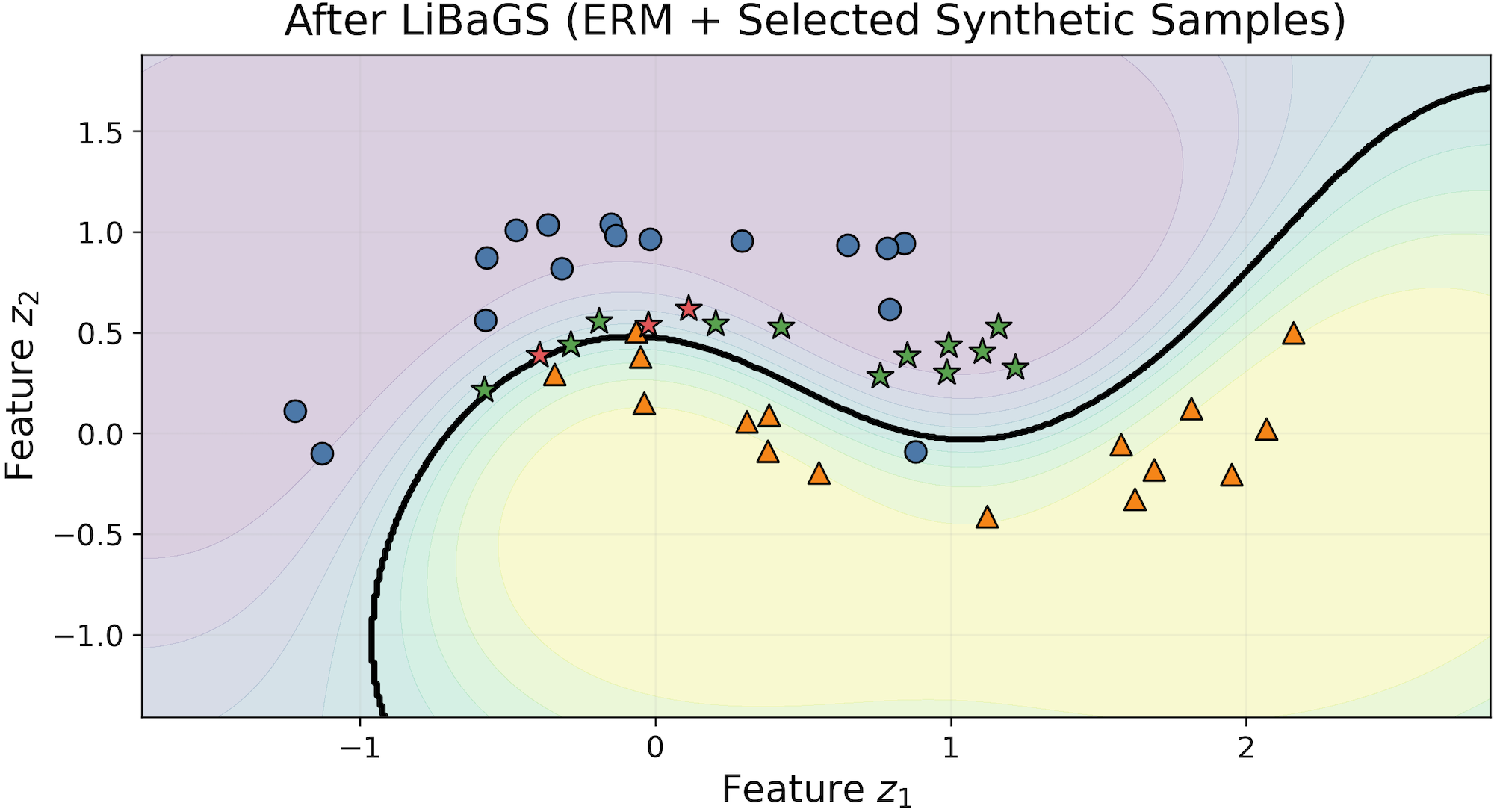} \\
    \vspace{1em}
    \includegraphics[width=0.38\linewidth]{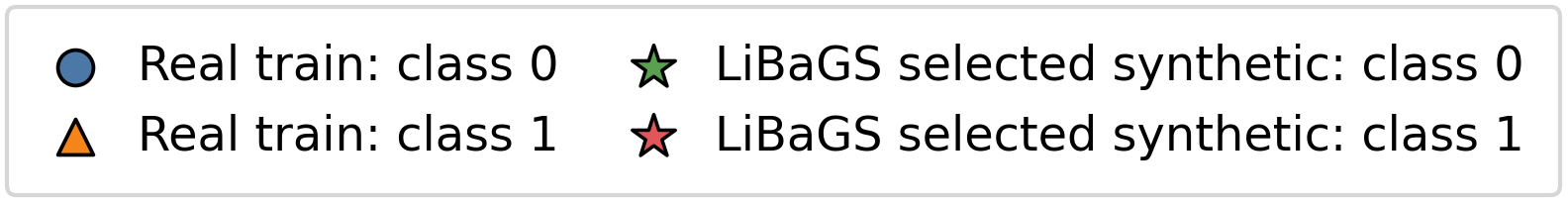}
    \caption{
Qualitative illustration of boundary-gap selection on the two-moons task. 
\textbf{Top}: with only a small real training set, ERM learns an imperfect decision boundary because several boundary-neighborhood regions are weakly covered by real samples. 
\textbf{Bottom}: LiBaGS identifies synthetic candidates with high boundary-gap value, combining boundary proximity, uncertainty, support validity, and low real-data coverage. The selected candidates, shown as green and red stars, are added to the real training set and produce a visibly improved decision boundary. The background shading represents the predicted class-probability surface.
}
    \label{fig}
\end{figure}

\section{Computational cost}
\label{compute}
Let $N$ be the number of real samples, $M$ the number of synthetic candidates, and $d$ the feature dimension. After candidate generation, LiBaGS has three main costs: fitting a lightweight scoring model on real features, scoring estimation for candidates, and diversity and support filtering, with the last step dominating the cost. The diversity and support filtering forms a candidate-candidate similarity matrix, costing $O(M^2d)$ time and $O(M^2)$ memory. Thus, after generation, the exact LiBaGS cost is $O(M^2d)$ for the asymptotic time complexity and $O(M^2)$ for the asymptotic space complexity. Empirically, for image data, generation still takes the most time and space. We used an NVIDIA A100 GPU with 40 GB of VRAM for the CIFAR-10 experiments, including image generation, which took about 1:20 minutes per seed run, given our run settings.

\section{Discussion \& Impact}
\label{impact}

LiBaGS can reduce the amount of real data needed in low-resource settings by selecting synthetic examples that target sparse boundary regions. This is especially useful when data collection is expensive, privacy-sensitive, time-consuming, or difficult to repeat. Rather than adding synthetic samples uniformly or relying only on uncertainty, LiBaGS attempts to identify candidates that are simultaneously near the decision boundary, plausible under the real data distribution, and located in regions with limited real-data coverage. This makes synthetic augmentation more targeted and can reduce unnecessary training on redundant or low-value generated samples.

The main risks come from synthetic data generation itself. If a generator reflects bias, label noise, spurious correlations, or unrealistic artifacts, LiBaGS may still select candidates that reinforce those problems. Support filtering reduces this risk but does not eliminate it, since a biased or incomplete real dataset can make biased synthetic samples appear on-support. Similarly, soft labels may reduce overconfidence near the boundary, but they cannot correct systematic errors in the generator or in the scoring model. As a result, LiBaGS should be viewed as a selection method rather than a guarantee of data quality.

In sensitive domains such as medicine, autonomous systems, education, etc., synthetic data can create a false sense of coverage. A selected synthetic sample may appear useful according to model-based scores while still being clinically implausible, demographically biased, or inconsistent with real deployment conditions. In such cases, LiBaGS should be paired with domain review, dataset audits, bias evaluation, external validation, and careful documentation of the generator, candidate pool, selected samples, and intended use.

\section{Limitations}
\label{limitations}
LiBaGS depends on the quality of the candidate pool. If the generator produces mostly easy examples, redundant examples, or off-support artifacts, the method may learn a small synthetic count or select candidates that provide only marginal gains. The method also relies on the representation map $h$ and the scoring model; poor representations can make $\hat p(z)$, $b(z)$, and $\pi(z)$ unreliable. Future work should compare the baselines in large-scale or multi-class regimes. We also use selection-signal baselines rather than reproducing every full external pipeline, which isolates the relevant selection criteria but may not capture all implementation details of those methods in detail.

\section{Conclusion}
LiBaGS is a lightweight method for targeted synthetic data selection. Rather than adding a fixed number of generated samples, it scores candidates by boundary importance, uncertainty, support validity, and real-data coverage, then stops when the estimated marginal gain becomes small. Across the two moons task, the 8$\times$8 digits 3 vs.\ 8 task, and the CIFAR-10 cat vs.\ dog experiment, LiBaGS improves over ERM and achieves the best mean accuracy among several selection baselines. The results support the claim that synthetic data is most useful when it fills sparse, on-support boundary gaps rather than merely increasing the size of the dataset.

\clearpage
\bibliographystyle{plainnat}
\bibliography{neurips_refs}

\clearpage
\appendix

\section{Technical appendices and supplementary material}
\label{appendix}
\label{app:proofs}

\subsection{Proof of Theorem~\ref{thm:local_alloc}}

\setcounter{theorem}{0}
\begin{theorem}[Local risk and boundary-gap allocation]
Assume that, in a small region around $z$, the local sample count is proportional to $np(z)+mq(z)$, where $p$ is the real density and $q$ is the selected synthetic density. If the local error decreases with this count and is weighted by this: $r(z)=a_\tau(z)u(z)b(z)$, then the allocation objective is:
\begin{equation}
J_m(q)=
\int_{\mathcal Z}
\frac{r(z)}{np(z)+mq(z)}
\,dz.
\end{equation}
For fixed $m>0$, for some $\lambda>0$, minimizing $J_m(q)$ over $q\geq0$ with $\int q(z)\,dz=1$ gives:
\begin{equation}
q^*(z)
=
\frac{1}{m}\left[
\sqrt{\frac{r(z)}{\lambda}}
-
np(z)
\right]_+.
\end{equation}
\end{theorem}

\begin{proof}
In a small region around $z$, the number of nearby training examples is proportional to $np(z)+mq(z)$. Local class-probability estimates become more stable as this count increases, so the local error is proportional to the inverse count. LiBaGS weights this local term using $r(z)=a_\tau(z)u(z)b(z)$, where $a_\tau(z)$ emphasizes boundary regions, $u(z)$ measures uncertainty, and $b(z)$ measures support. This gives the objective:
\begin{equation}
    J_m(q)=\int_{\mathcal Z}\frac{r(z)}{np(z)+mq(z)}\,dz.
\end{equation}

We optimize over valid synthetic allocation densities. The density $q(z)$ specifies where the fixed synthetic budget $m$ is placed in representation space. Thus $q(z)\geq 0$ and $\int q(z)\,dz=1$, while $mq(z)$ is the amount of synthetic mass assigned near $z$.

For a fixed $z$, define $\phi_z(q)=\frac{r(z)}{np(z)+mq}$. Since $\phi_z''(q)=\frac{2m^2r(z)}{(np(z)+mq)^3}\geq0$, each local term is convex in the allocation variable. Therefore, the full objective is convex in $q$.

Let $q^*$ be the optimal allocation. A region with $q^*(z)>0$ receives synthetic mass. Since the total synthetic budget is fixed, moving a small amount of mass between any two such regions should not reduce the objective; otherwise, $q^*$ would not be optimal. The derivative $\frac{d}{dq(z)}\phi_z(q)$ measures local change from adding synthetic mass near $z$. It is negative because adding mass increases the denominator $np(z)+mq(z)$ and decreases the local objective term. If two selected regions had different derivatives, shifting mass toward the region with the larger decrease would improve allocation. So, all regions with $q^*(z)>0$ must have the same derivative value. Hence, for some constant $C>0$,
\begin{equation}
    \frac{d}{dq(z)}
\left(
\frac{r(z)}{np(z)+mq(z)}
\right)
=
-C.
\end{equation}

Differentiating gives:
\begin{equation}
-\frac{mr(z)}{(np(z)+mq^*(z))^2}+C=0.
\end{equation}
Solving for the total local coverage gives:
\begin{equation}
np(z)+mq^*(z)
=
\sqrt{\frac{mr(z)}{C}}.
\end{equation}
Absorbing constants into $\lambda=\frac{C}{m}>0$ yields:
\begin{equation}
mq^*(z)
=
\sqrt{\frac{r(z)}{\lambda}}
-
np(z).
\end{equation}
Since synthetic density cannot be negative, negative allocations are clipped to zero:
\begin{equation}
q^*(z)
=
\frac{1}{m}\left[
\sqrt{\frac{r(z)}{\lambda}}
-
np(z)
\right]_+,
\end{equation}
where $\lambda>0$ is chosen so that the allocation satisfies $\int q^*(z)\,dz=1$.

Replacing $z$ with $z_j$ and replacing $p(z_j)$ with the empirical density estimate $\hat p(z_j)$ gives the finite-candidate score in Eq.~\eqref{eq:gap_score}.
\end{proof}

\subsection{Proof of Theorem~\ref{thm:adaptive}}

\begin{theorem}[Adaptive stopping]
For a small region $j$, let $c_j>0$ be real coverage, $t_j$ the number of selected synthetic samples, and $r_j\geq0$ the boundary importance. The gain from adding one more sample to region $j$ is:
\begin{equation}
\Delta_j(t_j)
=
\frac{r_j}{c_j+t_j}
-
\frac{r_j}{c_j+t_j+1}
=
\frac{r_j}{(c_j+t_j)(c_j+t_j+1)}.
\end{equation}
Greedy selection with threshold $\eta$ accepts gains $> \eta$ and stops when best remaining gain is $< \eta$.
\end{theorem}

\begin{proof}
For region $j$, adding one synthetic sample changes the local term from $\frac{r_j}{c_j+t_j}$ to $\frac{r_j}{c_j+t_j+1}$. The gain is therefore:
\begin{equation}
\Delta_j(t_j)
=
\frac{r_j}{c_j+t_j}
-
\frac{r_j}{c_j+t_j+1}
=
\frac{r_j}{(c_j+t_j)(c_j+t_j+1)}.
\end{equation}
As $t_j$ increases, the denominator increases, so the marginal gain decreases. Thus, each region has diminishing returns: the first synthetic samples assigned to an under-covered region are more useful than later ones assigned to the same region.

With some threshold $\eta > 0$ (where $\eta$ is chosen from the flattening point of the ordered marginal gain curve), adding a sample is worthwhile only if its gain is above $\eta$. Greedy selection always chooses the largest available gain across all regions. If this largest remaining gain is below $\eta$, then every other remaining gain is also below $\eta$, so no additional synthetic sample clears the stopping criterion.

Therefore, the greedy rule accepts exactly the samples whose estimated marginal gain is above $\eta$ and stops once all remaining gains are too small. This proves the adaptive stopping rule and shows why the selected count $\widehat m$ is determined by the data rather than fixed in advance.
\end{proof}

\subsection{Proof of Theorem~\ref{thm:soft}}

\begin{theorem}[Soft-label stability]
Let $\rho(z)=\Pr(Y\,|\,Z=z)$ be the true class distribution. The soft label $\widetilde y(z)=\bigl(1-a_\tau(z)\bigr)e_{c(z)}+a_\tau(z)\pi(z)$ satisfies:
\begin{equation}
\|\widetilde y(z)-\rho(z)\|_1
\leq
\bigl(1-a_\tau(z)\bigr)\|e_{c(z)}-\rho(z)\|_1
+
a_\tau(z)\|\pi(z)-\rho(z)\|_1.
\end{equation}
\end{theorem}

\begin{proof}
The soft label is:
\begin{equation}
    \widetilde y(z)
=
(1-a_\tau(z))e_{c(z)}
+
a_\tau(z)\pi(z).
\end{equation}

Subtracting the true class distribution $\rho(z)$ on both sides gives:
\begin{equation}
\widetilde y(z)-\rho(z)
=
(1-a_\tau(z))(e_{c(z)}-\rho(z))
+
a_\tau(z)(\pi(z)-\rho(z)).
\end{equation}
Taking the $\ell_1$ norm and applying the triangle inequality, we get:
\begin{equation}
\|\widetilde y(z)-\rho(z)\|_1
\leq
(1-a_\tau(z))\|e_{c(z)}-\rho(z)\|_1
+
a_\tau(z)\|\pi(z)-\rho(z)\|_1.
\end{equation}
\end{proof}

\subsection{Proof of Theorem~\ref{thm:diversity}}

\begin{theorem}[Diversity and support filtering]
Let $F(S)=\sum_{u\in\mathcal C}v_u\max_{j\in S}k(u,j)$ and also let $v_u=G_ub(z_u)$,
with $v_u\geq0$ and $k(u,j)\geq0$. Then, $F$ is monotone submodular.

Also, for a candidate $j$,
\begin{equation}
F(S\cup\{j\})-F(S)
\leq
\sum_{u\in\mathcal C}v_uk(u,j),
\end{equation}
so candidates with low support-weighted neighborhood value cannot pass the stopping threshold.
\end{theorem}

\begin{proof}
Let $A\subseteq B$, then for every $u$,
\begin{equation}
    \max_{j\in A}k(u,j)
\leq
\max_{j\in B}k(u,j).
\end{equation}

Since $v_u\geq0$, summing over $u$ gives $F(A)\leq F(B)$. Thus, $F$ is monotone. 

Let
$A\subseteq B$ and $e\notin A,B$. For every $u$, let $ m_A(u)=\max_{j\in A}k(u,j), 
m_B(u)=\max_{j\in B}k(u,j)$.

Since $A\subseteq B$, we have $m_A(u)\leq m_B(u)$ for every $u$.
The marginal gain of adding $i$ to $A$ is:
\[
F(A\cup\{i\})-F(A)
=
\sum_{u\in\mathcal U}
v_u
\left[
\max\{m_A(u),k(u,i)\}-m_A(u)
\right].
\]
Equivalently,
\[
F(A\cup\{i\})-F(A)
=
\sum_{u\in\mathcal U}
v_u
\left[
k(u,i)-m_A(u)
\right]_+,
\]
where $[x]_+$ clips negative values to 0. Similarly,
\[
F(B\cup\{i\})-F(B)
=
\sum_{u\in\mathcal U}
v_u
\left[
k(u,i)-m_B(u)
\right]_+.
\]
Because $m_A(u)\leq m_B(u)$, we have:
\[
\left[k(u,i)-m_A(u)\right]_+
\geq
\left[k(u,i)-m_B(u)\right]_+
\]
for every $u$. Since $v_u\geq 0$, summing over $u$ gives:
\[
F(A\cup\{i\})-F(A)
\geq
F(B\cup\{i\})-F(B).
\]
Therefore, $F$ is submodular.

So, $F$ is monotone submodular, i.e. adding candidates always improves coverage, but each additional candidate gives increasingly smaller gains once similar candidates have already been selected.

Finally, the marginal gain of adding $j$, as also shown above, is:
\begin{equation}
    F(S\cup\{j\})-F(S)
=
\sum_{u\in\mathcal C}
v_u
\left[
k(u,j)-\max_{\ell\in S}k(u,\ell)
\right]_+.
\end{equation}

Since $[x-y]_+\leq x$ for $x\geq0$,
\begin{equation}
    F(S\cup\{j\})-F(S)
\leq
\sum_{u\in\mathcal C}v_uk(u,j).
\end{equation}

Thus, adding candidate $j$ yields only a small gain when it is highly similar to candidates already contained in the selected set $S$.
\end{proof}

\end{document}